\documentclass{llncs}
\usepackage{fontenc}
\usepackage{hyperref}
\usepackage{graphicx}
\usepackage{stmaryrd}
\usepackage{subfig}
\usepackage{latexsym}
\usepackage{amsmath}
\usepackage{amssymb}
\usepackage{inputenc}
\usepackage{url}
\usepackage{titletoc}
\usepackage{tabularx}
\usepackage[cmyk]{xcolor}
\usepackage{stmaryrd}

\newcommand{\ifof}{if and only if }
\newcommand{\rsd}{rough semantic domain }
\newcommand{\rsds}{rough semantic domains }
\newcommand{\ifsf}{\,\mathrm{if}\: \mathrm{and}\:\mathrm{only}\: \mathrm{if}\: }

\newenvironment{mitemize}{\begin{itemize}}{\end{itemize}}
\newenvironment{impemize}{\begin{itemize}}{\end{itemize}}

\begin{document}

\title{More on Dialectics of Knowledge Representation in a Granular Rough Set Theory}
\author{A. Mani\thanks {I would like to thank Prof Mihir Chakraborty for
discussions on PRAX.}}
\institute{Department of Pure Mathematics\\
University of Calcutta\\
9/1B, Jatin Bagchi Road\\
Kolkata-700029, India\\
Web: \url{http://www.logicamani.in}\\
Email: \texttt {$a.mani.cms@gmail.com$}\\
\textbf{\textsc{Expanded Version of ICLA'2013 Paper}}}
\maketitle

\begin{abstract}
The concepts of rough and definite objects are relatively more determinate than
those of granules and granulation in general rough set theory (\textsf{RST})
\cite{AM240}. Representation of rough objects can however depend on the dialectical
relation between granulation and definiteness. In this research, we exactify this aspect
in \textsf{RST} over proto-transitive approximation spaces. This approach can
be directly extended to many other types of \textsf{RST}. These are used for formulating
an extended concept of knowledge interpretation (\textsf{KI})(relative the situation for
classical
\textsf{RST}) and the problem of knowledge representation (\textsf{KR}) is solved. These
will be
of direct interest in granular \textsf{KR} in \textsf{RST} as developed by
the present author \cite{AM909} and of rough objects in general. In \cite{AM2400}, these
have already been used for five different semantics by the present author. This is an extended version of \cite{AM270} with key examples and more results. 

\medskip

\textbf{keywords}:
Rough Objects, Proto Transitivity, Granulation, RYS, Lattice Theory of Definite Objects,
Axiomatic Theory of Granules, Contamination Problem, Granular Knowledge, Knowledge
Representation.
\end{abstract}


\section{Introduction}

General transitive relations are of much interest in a wide variety of application
scenarios including vagueness and preference. But the semantics for the former cases
have not been considered in the literature prior to this and \cite{AM2400} by the
present author. Here we develop the \textsf{KI} over a granular
representation of rough objects. We also show the gaping holes in meaning
that would be left by omitting granularity in the considerations.

Rough objects as explained in \cite{AM99,AM240} are collections of objects in a classical
domain (Meta-C) that appear to be indistinguishable among themselves in another rough
semantic domain (Meta-R). But their representation in most \textsf{RST}s in purely
order theoretic terms is not known. In this research paper, we do this for a specific type
of \textsf{RST} over proto-transitive approximation spaces (\textsf{PRAS}) developed
recently by the present author in \cite{AM2400}. The method can be directly extended to
many other types of \textsf{RST}. In \cite{AM2400}, five different algebraic semantics in
slightly different \rsds are also developed by the present author. In general, rough
objects correspond to concepts in the Pawlak-sense \textsf{KI} as
they correspond to
the \emph{objects that can be perceived in the \rsd}. Here we focus on the
generalized approach to granular \textsf{KI} in general \textsf{RST}
initiated in \cite{AM909}. Relative the extended \textsf{KI} in a semantic
domain placed between the classical (Meta-C) and Meta-R (of rough objects), the problem of
\textsf{KR} is solved for reflexive \textsf{PRAS} (\textsf{PRAX}). 

From a simplified view, rough objects may be represented by pairs or tuples of definite
objects under some conditions. In classical \textsf{RST}, definite objects are precisely
those that satisfy $x^l \,=\, x^u \,=\, x$. Pairs of definite objects of the form $(a,
b)$ satisfying $a \subseteq b$ necessarily represent rough objects and every rough object
can be so represented. Further intervals of the form $]a, b[\,=\, \{c\,: \, a\subset
\, c\, \subset b,\}$ represent rough objects if and only if $b$ covers $a$. However not
all intervals of the form correspond to rough objects without the covering condition.
This relation becomes more complicated in more general \textsf{RST}s as in the case of
\textsf{PRAX}. We characterize this in this research.  

Classical \textsf{RST} is \textsf{RST} starting from an
approximation space of the form $\left\langle \underline{S},\, R \right\rangle$ with
$\underline{S}$ being a set and $R$ being an equivalence relation on $\underline{S}$. 
\emph{Weak transitivity} of \cite{ICH} is \emph{proto-transitivity} here and distinct
from the usage in \cite{AM24}.

\begin{definition}
A binary relation $R$ on a set $S$ is said to be \emph{weakly-transitive, transitive or
proto-transitive} respectively on $S$ \ifof 
$S$ satisfies 
\begin{mitemize}
\item {$(\forall x, y,
z)(Rxy,\,Ryz \,\&\, x\,\neq\,y\,\neq\,z\,\longrightarrow\,Rxz)$
(i.e. $(R\circ R)\setminus \Delta_{S}\,\subseteq R$ (where $\circ$ is relation composition)
, or}
\item {$(\forall x, y,
z)(Rxy \,\&\,Ryz \,\longrightarrow\,Rxz)$ (i.e. $(R\circ
R) \subseteq R$), or }
\item {$(\forall x, y,
z)(Rxy,\,Ryz , \,Ryx ,\, Rzy \,\&\,
x\,\neq\,y\,\neq\,z\,\longrightarrow\,Rxz )$, respectively. Proto-transitivity of
$R$ is equivalent to $R\cap R^{-1} \,=\,\tau(R)$ being weakly transitive.} 
\end{mitemize}
\end{definition}

$Ref (S)$, $r\tau (S),$ $w\tau(S),\, p\tau(S),\, EQ(S)$ will respectively denote the set
of reflexive, transitive, weakly transitive, proto transitive,
and equivalence relations on the set $S$ respectively. 
\begin{proposition}
In general, $w\tau(S) \,\subseteq\, p\tau(S) .$
If $R\in Ref(S)$, then $R\in p\tau(S)$ \ifof $\tau(R)\in EQ(S)$. 
\end{proposition}

\subsection{Granules and Granular Computing Paradigms}

The idea of granular computing is as old as human evolution. Even in the available information on earliest human habitations and dwellings, it is possible to identify a primitive granular computing process (\textsf{PGCP}) at work. This can for example be seen from the stone houses, dating to 3500 BCE, used in what is present-day Scotland. The main features of this and other primitive versions of the paradigm may be seen to be   

\begin{itemize}
\item {Problem requirements are not rigid.}
\item {Concept of granules may be vague.}
\item {Little effort on formalization right up to approximately the middle of the previous century.}
\item {Scope of abstraction is very limited.}
\item {Concept of granules may be concrete or abstract (relative all materialist viewpoints).}
\end{itemize}

The precision based granular computing paradigm, traceable to Moore and Shannon's paper \cite{Sha56}, will be referred to as the \emph{classical granular computing paradigm} \textsf{CGCP} is usually understood as the granular computing paradigm (The reader may  note that the idea is vaguely present in \cite{Sha48}). The distinct terminology would be useful to keep track of the differences with other paradigms. CGCP has since been adapted to fuzzy and rough set theories in different ways. 

Granules may be assumed to subsume the concept of information granules -- information at some level of precision. In granular approaches to both rough and fuzzy sets, we are usually concerned with such types of granules. Some of the fragments involved in applying CGCP may be:
\begin{itemize}
\item {Paradigm Fragment-1: Granules can exist at different levels of precision.}
\item {Paradigm Fragment-2: Among the many precision levels, choose a precision level at which the problem at hand is solved. }
\item {Paradigm Fragment-3: Granulations (granules at specific levels or processes) form a hierarchy (later development).}
\item {Paradigm Fragment-4: It is possible to easily switch between precision levels.}
\item {Paradigm Fragment-5: The problem under investigation may be represented by the hierarchy of multiple levels of granulations.}
\end{itemize}

The not so independent stages of development of the different granular computing paradigms is stated below:  
\begin{itemize}
\item {Classical Primitive Paradigm till middle of previous century.}
\item {CGCP: Since Shannon's information theory}
\item {CGCP in fuzzy set theory. It is natural for most real-valued types of fuzzy sets, but even in such domains unsatisfactory results are normal. For one thing linguistic hedges have little to do with numbers. A useful reference would be \cite{LZ9}.}
\item {For a long period (up to 2008 or so), the adaptation of CGCP for RST has been based solely on precision and related philosophical aspects. The adaptation is described for example in \cite{Ya01}. In the same paper the hierarchical structure of granulations is also stressed. This and many later papers on CGCP (like \cite{TYL}) in rough sets speak of structure of granulations. }
\item {Some Papers with explicit reference to multiple types of granules from a semantic viewpoint include \cite{AM69,AM99,SW3,SW,AM105}.}
\item {The axiomatic approach to granularity initiated in \cite{AM99} has been developed by the present author in the direction of contamination reduction in \cite{AM240}. From the order-theoretic/algebraic point of view, the deviation is in a very new direction relative the precision-based paradigm. The paradigm shift includes a new approach to measures. }
\end{itemize}

There are other adaptations of CGCP to soft computing like \cite{KCM} that we will not consider.

Unless the underlying language is restricted, granulations can bear upon the theory with unlimited diversity. Thus for example in classical \textsf{RST}, we can take any of the following as granulations: collection of equivalence classes, complements of equivalence classes, other partitions on the universal set $S$, other partition in $S$, set of finite subsets of $\mathcal{S}$ and set of finite subsets of $\mathcal{S}$ of cardinality greater than 2. This is also among the many motivations for the axiomatic approach. 

A formal simplified version of the the axiomatic approach to granules is in \cite{AM1800}. The axiomatic
theory is capable of handling most contexts and is intended to permit relaxation of
set-theoretic axioms at a later stage. The axioms are considered in the framework of Rough
Y-Systems (\textsf{RYS}) that  maybe seen as a generalized form of \emph{abstract
approximation spaces} \cite{CC5} and approximation framework
\cite{CD3}. It includes relation-based RST, cover-based RST and more. These structures are
provided with enough structure so that a classical semantic domain (Meta-C) and at least
one rough semantic domain (called Meta-R) of roughly equivalent objects along with
admissible operations and predicates are associable. But the exact way of association is
not something absolute as there is no real end to recursive approximation processes of
objects. 

Paradigms do matter in the granular knowledge context in a highly non trivial way.

\section{Motivation and Examples}

Generalized transitive relations occur frequently in general information systems,
but are often not recognized as such and there is hope for improved semantics and
\textsf{KI} relative the situation for purely reflexive relation based \textsf{RST}. Not
all of the definable approximations have been investigated in even closely related
structures of general \textsf{RST}. Contamination-free semantics \cite{AM240} for the
contexts are also not known. Finally these relate to \textsf{RYS} and variants. A proper
characterization of roughly equal (requal) objects is also motivated by \cite{AM240}.

\subsection{Example-1}
Let $\S\,=\, \{a, b, c, e, f, g, h, l, n\}$ and let $R$  be a binary relation on it 
defined via\\ $R\,=\, \{(a,\,a),\,
(l,\,l),\,(n,\,n),\,(n,\,h),\,(h,\,n),\,(l,\,n),\,(g,\,c),\,(c,\,g),$\\$\,(g,\,l),\,(b,\,
g), \,(g,\,b ), \, (h, \,g ), \,(a,\,b),\,(b,\,c),\,(h,\,a),\,(a,\,c)\} $.
Then $\left\langle S,\, R\right\rangle $ is a \textsf{PRAS}.
 
If $P$ is the reflexive closure of $R$ (that is
$P\,=\, R\cup \Delta_{S}$), then $\left\langle S,\, P\right\rangle $ is a \textsf{PRAX}.
The successor neighbourhoods associated with different elements of $S$ are as follows
(\textbf{E} is a variable taking values in $S$):

\begin{tabular}{|c|c|c|c|c|c|c|c|c|c|}
\hline
\hline
\textbf{E} & $a\,$ & $ b\,$ & $c\,$ & $g\,$ & $e\,$ & $f \,$ & $h\,$ & $l\,$ & $n\,$\\
\hline
\textbf{$[E]$} & $\{a, h \}$ & $\{b, c, g \}$ & $\{b, c, g \}$ & $\{b, c,
g, h \}$ &$\{e \}$ & $\{ f\}$ & $\{h,n \}$ & $\{l,g \}$ & $\{n,l,g,h\}$\\
\hline
\textbf{$[E]_{o}$} & $\{ a\}$ & $\{b, c, g \}$ & $\{ b, c, g\}$ &$\{b,
c, g \}$ & $\{e\}$ & $\{ f\}$ & $\{h, n \}$ & $\{ l\}$ &$\{n,h \}$\\
\hline
\hline
\end{tabular}

If $A\,=\,\{a,\, h,\, f\}$, then \[A^{l}\,=\,\{a,\, h,\, f\},\; \mathrm{while}\,
A^{lo}\,=\,\{a,\, f\}\, \mathrm{and}\;A^{lo}\,\subset \, A^{l}.\]

If $F\,=\,\{l\}$, then \[F^{l}\,=\,\emptyset,\; \mathrm{while} F^{lo}\,=\,F\;
\mathrm{and}\; F^{l}\,\subset \, F^{lo}\,.\]

Now let $Z\,=\, N \cup S \cup X$, where $N$ is the set of naturals, $X$ is the set of
elements of the infinite sequences $\{x_{i}\},\,\{y_j\} $. Let $Q$ be a relation on $Z$
such that $Q\cap S^{2} = P$, $Q\cap N^{2}$ be some equivalence. Further let\\ $(\forall i
\in N) (i, x_{3i +1}),\, (x_{2i}, i), \, (x_{i}, x_{i+1}),\,(y_{i}, y_{i+1})\in Q $.  
For any $i\in N$, let $P_{i}\,=\,\{y_{k}:\,k\neq 2j \& k < i \}\,\cup\,\{x_{2j}:\, 2j <
i\}$ - this will be used in later sections. The extension of the example to involve nets
and densely ordered subsets is standard.

\subsection{Caste Hierarchies and Interaction}

The caste system and religion are among the deep-seated evils of Indian society that
often cut across socio-economic classes and level of education. For the formulation of
strategies aimed at large groups of people towards the elimination of such evils it would
be fruitful to study interaction of people belonging to different castes and religions on
different social fronts.  

Most of these castes would have multiple subcaste hierarchies in addition.
Social interactions are necessarily constrained by their type and untouchability
perception. If $x, \, y$ are two castes, then with respect to a possible social
interaction $\alpha$, people belonging to $x$ will either regard people belonging to $y$
as untouchable or otherwise. As the universality is so total, it is possible to write
$\mathbb{U}_{\alpha}xy$ to mean that $y$ is untouchable for $x$ for the interaction
$\alpha$. Usually this is a assymmetric relation and $y$ would be perceived as a
\emph{lower caste} by members of $x$ and many others.

Other predicates will of course be involved in deciding on the possibility of the social
interaction, but if $\mathbb{U}_{\alpha}xy$ then the interaction is forbidden relative
$x$. If $\alpha$ is "context of possible marriage", then the complementary relation
($\mathbb{C}_{\alpha}$ say) is a reflexive proto-transitive relation. For various other
modes of interaction similar relations may be found.

In devising remedial educational programmes targeted at mixed groups, it would be important
to understand approximate perceptions of the group and the semantics of PRAX
would be very relevant.

\section{Approximations Unlimited in a PRAX}

\begin{definition}
By a \emph{Proto Approximation Space} $S$ (\textsf{PRAS for short}), we will mean a pair
of the form $\left\langle \underline{S},\, R \right\rangle  $ with $\underline{S}$ being a
set and $R$ being a proto-transitive relation on it. If $R$ is also reflexive, then it
will be called a \emph{Reflexive Proto Approximation Space} (\textsf{PRAX}) for short).
In general, we will assume that $\underline{S}$ is infinite.
\end{definition}

If $S$ is a PRAX or a PRAS, then we will respectively denote successor neighbourhoods,
and symmetrized successor neighbourhoods generated by an
element $x \in S$ as follows:
\[[x]\,=\, \{y ;\, Ryx\}\;\; [x]_{o}\,=\, \{y ;\, Ryx\, \&\, Rxy\}.\]

\begin{definition}
If $A\subseteq S$ is an arbitrary subset of a  \textsf{PRAX} or a
\textsf{PRAS} $S$, then let
\begin{description}
\item [Upper Proto: ]{$A^{u}\,=\, \bigcup _{[x]\cap A \neq \emptyset} {[x]}.$} 
\item [Lower Proto: ]{$A^{l}\,=\, \bigcup_{[x]\subseteq A} {[x]}$.}
\item [Symmetrized Upper Proto]{$A^{uo}\,=\, \bigcup _{[x]_{o}\cap A \neq \emptyset}
{[x]_{o}}.$}
\item [Symmetrized Lower Proto]{$A^{lo}\,=\, \bigcup_{[x]_{o}\subseteq A} {[x]_{o}}. $}
\item [Point-wise Upper]{$A^{u+}\,=\, \{ x \, :\, [x]\cap A \neq \emptyset \}.$}
\item [Point-wise Lower]{$A^{l+}\,=\, \{x \,:\,[x]\subseteq A \}\,. $}
\end{description}
\end{definition}

\begin{proposition}
In a PRAX $S$ and for a subset $A\subseteq S$, all of the following hold:
\begin{mitemize}
\item {$(\forall x) \,[x]_{o}\subseteq [x]$ }
\item {It is possible that $A^{l}\,\neq \, A^{l+}$ and in general, $A^{l}\,\parallel \,
A^{lo}$.}
\end{mitemize}
\end{proposition}
\begin{proof}
The proof of the first two parts are easy. For the third, we chase the argument up to a
trivial counter example. We have already provided a counter example for this. 
\[\bigcup_{[x]\subseteq A} {[x]}\,\subseteq \bigcup_{[x]_{o}\subseteq A} {[x]}\, \supseteq
\,\bigcup_{[x]_{o}\subseteq A} {[x]_{o}} \]
\[\bigcup_{[x]_{o}\subseteq A} {[x]_{o}} \supseteq  \bigcup_{[x]\subseteq A} {[x]_{o}}\,
\subseteq \,\bigcup_{[x]\subseteq A} {[x]}. \]\qed
\end{proof}

\begin{proposition}
For any subset $A$ of a \textsf{PRAX} $S$, $A^{uo}\,\subseteq\,
A^{u}$. 
\end{proposition}

\begin{proof}
Since $[x]_{o}\cap A\neq \emptyset$, therefore
\[A^{uo}\,=\, \bigcup _{[x]_{o}\cap A \neq \emptyset} {[x]_{o}} \,\subseteq \, \bigcup
_{[x] \cap A \neq \emptyset} {[x]_{o}}\, \subseteq\,A^{uo}\,=\, \bigcup _{[x]\cap A \neq
\emptyset} {[x]}\,=\, A^{u}.\]
\qed
\end{proof}

\begin{definition}
In a PRAX $S$, if $X$ is an approximation operator, then by a $X$-\emph{definite element},
we will mean a subset $A$ satisfying $A^{X}\,=\, A$. The set of all $X$-definite elements
will be denoted by $\delta_{X}(S)$, while the set of $X$ and $Y$-definite elements ($Y$
being another approximation operator) will be denoted by $\delta_{XY}(S)$. 
\end{definition}

\begin{definition}
In a \textsf{PRAX} $S$, a subset $A$ will be said to be respectively \emph{lower
proto-definite} or \emph{upper proto definite} if and only if $A^{l}\,=\, A$ or
$A^{u}\,=\, A$ holds respectively. $A$ will be said to be \emph{proto-definite} if and
only if it is both upper and lower proto-definite. 
\end{definition}

\begin{proposition}
In a \textsf{PRAX} $S$,	$\delta_{u}(S)$ is a complete sublattice of $\wp (S)$ with respect
to inclusion.
\end{proposition}

\begin{proof}
As $R$ is reflexive, if $A,\, B $ are upper proto definite, then $A\cup B$ and $A\cap B$
are both upper proto definite. So the result holds.  
\qed
\end{proof}

\begin{proposition}
In a \textsf{PRAX} $S$, $\delta_{lo}(S)\,=\, \delta_{uo}(S)$.
\end{proposition}

\begin{theorem}
In a \textsf{PRAX} $S$, the following hold:
\begin{mitemize}
\item {$\delta_{u}(S)\,\subseteq\,\delta_{uo}(S)  $. }
\item {$\delta_{l}(S)\,\parallel\,\delta_{lo}(S)$ in general. }
\end{mitemize}
\end{theorem}

\begin{proof}
If $A\,\in\,\delta_{u} $, then $(\forall x \in A ) [x]\subseteq A$ and 
$(\forall x \in A^{c}) [x]\cap A = \emptyset$.

So $(\forall x \in A^{c})\, [x]_{o}\cap A = \emptyset $.
But as $A\,\subseteq\,A^{uo}$ is necessary, we must have $A\in \delta_{uo}$.
\qed
\end{proof}

\begin{theorem}
In a \textsf{PRAS} $S$, it is possible that $\delta_{u}\,\nsubseteq \delta_{uo}$.
\end{theorem}

\begin{remark}
$A^{u+},\, A^{l+}$ have relatively been more commonly used in the literature and 
have also been the only kind of approximation studied in \cite{JJ} for example (the
inverse relation is also considered from the same perspective). The main
features of these is stated in the following theorem. 
\end{remark}

\begin{theorem}
If $u+,\, l+ $ are treated as self maps on the power-set $\wp (S)$, $S$ being a PRAX or a
\textsf{PRAS} then all of the following hold:
\begin{mitemize}
\item {$(\forall x)\, x^{cl+}=x^{u+c},\,\,x^{cu+}=x^{l+c} $ - that is $l+$ and $u+$ are
mutually dual} 
\item {$l+,\, u+$ are monotone.}
\item {$l+$ is a complete $\wedge$-morphism, while $u+$ is a $\vee$-morphism.}
\item {$\partial (x) = \partial(x^{c})$, where partial stands for the boundary operator.}
\item {$\Im (u+)$ is an interior system while $\Im (l+)$ is a closure system.}
\item {$\Im (u+)$ and $\Im (l+)$ are dually isomorphic lattices.}
\end{mitemize}
\end{theorem}

\begin{theorem}
In a \textsf{PRAX} $S$, $(\forall A)\, A^{l+} \subseteq A^{l},\:\:A^{u+} \subseteq A^{u}
$. 
\end{theorem}

\begin{proof}
\begin{mitemize}
\item {If $x\in A^{l+}$, then $[x]\subseteq A$ and so $[x]\subseteq A^{l},\, x\in A^{l}$.}
\item {If $x\in A^{l}$, then $(\exists y \in A) [y]\subseteq A,\, Rxy$. But it is possible
that $[x] \nsubseteq A$, therefore it is possible that $x\notin A^{l+}$ and
$A^{l}\nsubseteq A^{l+}$.}
\item {If $x\in A^{u+}$, then $[x]\cap A \neq \emptyset$, so $x\in A^{u}$.}
\item {So $A^{u+}\subseteq A^{u}$.}
\item {Note that $x\in A^{u}$, \ifof $(\exists z\in S) \,x\in [z],\, [z]\cap A\neq
\emptyset $, but this does not imply $x\in A^{u+}$. }
\end{mitemize}
\end{proof}

\begin{theorem}
In a \textsf{PRAX} $S$, $(\forall A, B \in \wp(S))\, A^{l}\cup B^{l}\,\subseteq (A\cup
B)^{l}$. 
\end{theorem}

\begin{proof}
For any $A, B\in \wp{S}$, $x\in (A\cup B)^{l}$
\begin{impemize}
\item {$(\exists y\in (A\cup B) )\, x\in [y]\, \subseteq \, A\cup B $.}
\item {$(\exists y\in A )\, x\in [y]\, \subseteq \, A\cup B $ or $(\exists y\in  B
)\, x\in [y]\, \subseteq \, A\cup B $.}
\item {$(\exists y\in A) \, x\in [y]\, \subseteq \, A $ or $(\exists y\in A) \, x\in [y]\,
\subseteq \, B $ or $(\exists y\in B) \, x\in [y]\, \subseteq \, A $ or $(\exists y\in B)
\, x\in [y]\, \subseteq \, B $}
\end{impemize}
 
Clearly the last statement is implied by $x\in A^{l} \cup B^{l}$.
\end{proof}

\begin{theorem}
In a \textsf{PRAX} $S$,  $(\forall A, B \in \wp(S))\, (A\cap
B)^{l}\,\subseteq\, A^{l}\cap B^{l}$. 
\end{theorem}

\begin{proof}
For any $A, B\in \wp{S}$, $x\in (A\cap B)^{l}$
\begin{impemize}
\item {$ x\in \, A\cap B $}
\item {$(\exists y\in A\cap B )\, x\in [y]\, \subseteq \, A\cap B $ and $x\,\in\, A,\;
x\,\in\, B$}
\item {$(\exists y\in A) \, x\in [y]\, \subseteq \, A $ and $(\exists y\in B) \, x\in
[y]\,\subseteq \,B $}
\end{impemize}
Clearly the last statement implies  $x\in A^{l} \& x\in  B^{l}$, but the converse is not
true in general. 
\qed
\end{proof}
 
\begin{theorem}
In a \textsf{PRAX} $S$, all of the following hold:
\begin{enumerate}
\item {$(\forall A, B\in \wp (S))\, (A\cap B)^{l+}\,=\, A^{l+}\cap B^{l+}$}
\item {$(\forall A, B\in \wp (S))\, A^{l+}\cup B^{l+}\, \subseteq (A\cup B)^{l+}$}
\item {$(\forall A \in \wp (S))\, (A^{l+})^{c} = (A^{c})^{u+} $} 
\end{enumerate}
\end{theorem}

\begin{proof}
\begin{enumerate}
\item {$x\in (A\cap B)^{l+}$ 
 \begin{impemize}
 \item {$[x]\subseteq A\cap B $} 
 \item {$[x]\subseteq A $ and $[x]\subseteq B $}
 \item {$x\in xA^{l+} $ and $x\in B^{l+}$.}
 \end{impemize}}
\item {$x \in A^{l+}\cup B^{l+} $
  \begin{impemize}
  \item {$[x]\subseteq A^{l+} $ or $[x]\subseteq B^{l+}$}
  \item {$[x]\subseteq A $ or $[x]\subseteq B$}
  \end{impemize}
$\Rightarrow [x]\subseteq A\cup B\, \Leftrightarrow\, x\in (A\cup B)^{l+} $.}
\item {$z\in A^{l+c}$
 \begin{impemize}
 \item {$z\,\notin\, A^{l+}$}
 \item {$[z]\,\nsubseteq A $}
 \item {$z\cap A^{c}\,\neq \,\emptyset $}
 \end{impemize}
}
\end{enumerate}
\qed
\end{proof}

\begin{theorem}
In a \textsf{PRAX} $S$, all of the following hold:
\begin{enumerate}
\item {$(A\cup B)^{u}\,=\, A^{u}\cup B^{u}$}
\item {$(A\cap B)^{u} \,\subseteq\, A^{u}\cap B^{u}$} 
\end{enumerate}
\end{theorem}

\begin{proof}
\begin{enumerate}
\item {$x\,\in\, (A\cup B)^{u}$
 \begin{impemize}
  \item {$x\,\in\, \bigcup_{[y]\cap (A\cup B)\neq \emptyset } [y] $}
  \item {$x\,\in\, \bigcup_{([y]\cap A)\cup ([y]\cap B)\neq \emptyset}  $}
  \item {$x\,\in\,\bigcup_{[y]\cap A\neq \emptyset} [y]  $ or $x\,\in\,\bigcup_{[y]\cap
B\neq \emptyset} [y]  $}
  \item {$x\in A^{u}\cup B^{u}$.}
 \end{impemize}
}
\item {By monotonicity, $(A\cap B)\,\subseteq\, A^{u}$ and $(A\cap B)\,\subseteq\,
B^{u}$, so $(A\cap B)^{u}\,\subseteq\, A^{u}\cap B^{u} $.
} 
\end{enumerate}
\qed
\end{proof}

\begin{theorem}
In a \textsf{PRAX} $S$, all of the following hold:
\begin{mitemize}
\item {$(\forall A\in \wp (S))\, A^{l+}\,\subseteq\, A^{lo} $}
\item {$(\forall A\in \wp (S))\, A^{uo}\,\subseteq\, A^{u+} $}
\end{mitemize}
\end{theorem}

\begin{theorem}
In a \textsf{PRAX} $S$, \[(\forall A\in \wp (S))\, A^{l+}\subseteq A^{lo}\] 
\end{theorem}

\begin{proof}
If $x\in A^{l+}$, then $[x]\subseteq A$.
But as $[x]_{o}\subseteq [x]$, $A^{l+}\subseteq A^{lo}$.
\end{proof}

\begin{theorem}
In a \textsf{PRAX} $S$, \[(\forall A\in \wp (S))\, A^{lc}\,\subseteq\, A^{cu} .\]
\end{theorem}

\begin{proof}
If $z\in A^{lc}$, then $ z\in [x]^{c}$ for all $[x]\subseteq A$ and either, $z\in
A\setminus A^{l}$ or $z\in A^{c}$.

If $z\in A^{c}$ then $z\in A^{cu}$.

If $z\in A\setminus A^{l}$ and $z\neq A^{cu \setminus A^{c}}$ then $[z]\cap A^{c}\,=\,
\emptyset$. But this contradicts $z\notin A^{cu} \setminus A^{c}$.

So $(\forall A\in \wp (S))\, A^{lc}\,\subseteq\, A^{cu} .$
\qed
\end{proof}

From the above, we have the following relation between approximations in general ($u+
\longrightarrow u$ should be read as \emph{the $u+$- approximation of a set is included
in the $u$-approximation of the same set}):

\includegraphics[width=11.4cm]{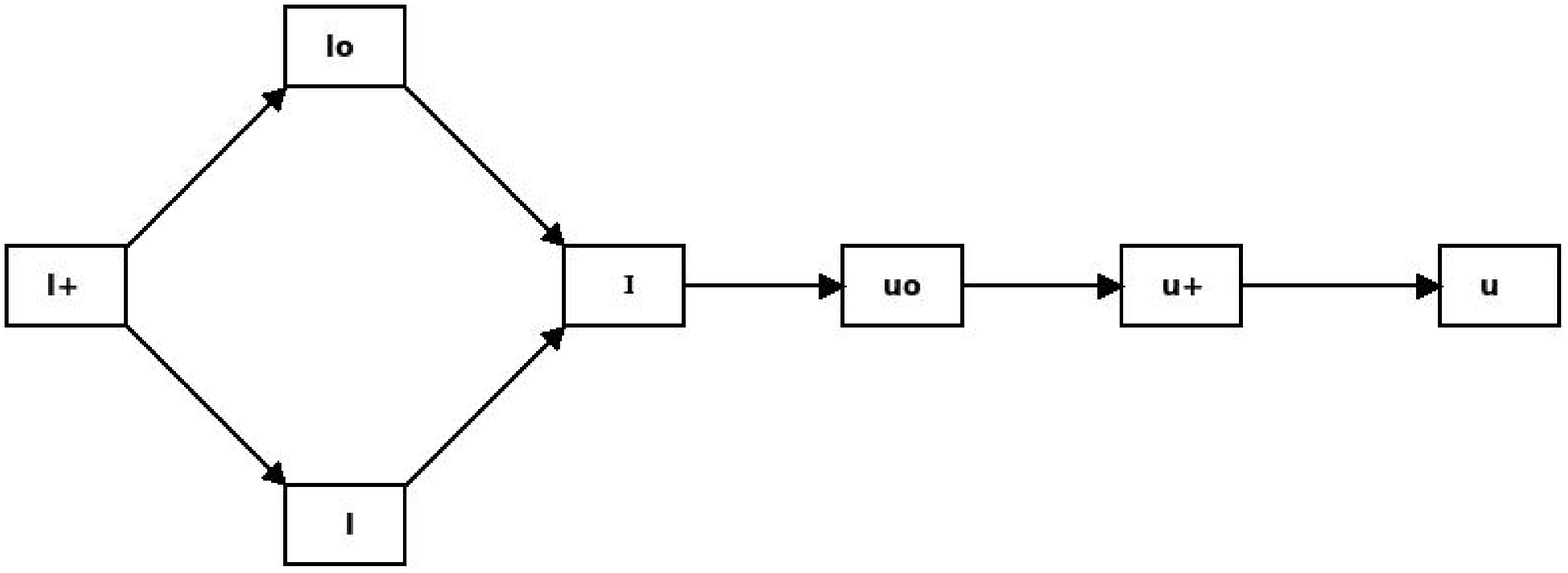}

If a relation $R$ is purely reflexive and not proto-transitive on a set $S$, then the
relation $\tau(R)\,=\, R\cap R^{-1}$ will not be an equivalence and for a $A\subset S$,
it is possible that $A^{uo l}\subseteq A $ or $A^{uol}\parallel A$ or $A\subseteq
A^{uol}$.

\begin{theorem}
In a \textsf{PRAX}, all of the following hold: 
\begin{description}
\item [Bi]{$(\forall A\in \wp(S))\, A^{ll}\,=\, A^{l}\, \& \, A^{u} \subseteq\, A^{uu} $.}
\item [l-Cup]{$(\forall A, B \in \wp(S))\, A^{l}\cup B^{l}\,\subseteq (A\cup B)^{l}$.}
\item [l-Cap]{$(\forall A, B \in \wp(S))\, (A\cap B)^{l}\,\subseteq\, A^{l}\cap B^{l}$.}
\item [u-Cup]{$(\forall A, B \in \wp(S))\,(A\cup B)^{u}\,=\, A^{u}\cup B^{u}$}
\item [u-Cap]{$(\forall A, B \in \wp(S))\,(A\cap B)^{u} \,\subseteq\, A^{u}\cap B^{u}$} 
\item [Dual]{$(\forall A\in \wp (S))\, A^{lc}\,\subseteq\, A^{cu}$.}
\end{description}
\end{theorem}
\begin{proof}
\begin{description}
\item[l-Cup]{For any $A, B\in \wp{S}$, $x\in (A\cup B)^{l}$
\begin{impemize}
\item {$(\exists y\in (A\cup B) )\, x\in [y]\, \subseteq \, A\cup B $.}
\item {$(\exists y\in A )\, x\in [y]\, \subseteq \, A\cup B $ or $(\exists y\in  B
)\, x\in [y]\, \subseteq \, A\cup B $.}
\item {$(\exists y\in A) \, x\in [y]\, \subseteq \, A $ or $(\exists y\in A) \, x\in [y]\,
\subseteq \, B $ or $(\exists y\in B) \, x\in [y]\, \subseteq \, A $ or $(\exists y\in B)
\, x\in [y]\, \subseteq \, B $ - this is implied by $x\in A^{l} \cup B^{l}$.}
\end{impemize}}
\item[l-Cap]{For any $A, B\in \wp{S}$, $x\in (A\cap B)^{l}$
\begin{impemize}
\item {$ x\in \, A\cap B $}
\item {$(\exists y\in A\cap B )\, x\in [y]\, \subseteq \, A\cap B $ and $x\,\in\, A,\;
x\,\in\, B$}
\item {$(\exists y\in A) \, x\in [y]\, \subseteq \, A $ and $(\exists y\in B) \, x\in
[y]\,\subseteq \,B $ - Clearly this statement implies  $x\in A^{l} \& x\in  B^{l}$,
but the converse is not true in general.}
\end{impemize}} 
\item[u-Cup]{$x\,\in\, (A\cup B)^{u}$
 \begin{impemize}
  \item {$x\,\in\, \bigcup_{[y]\cap (A\cup B)\neq \emptyset } [y] $}
  \item {$x\,\in\, \bigcup_{([y]\cap A)\cup ([y]\cap B)\neq \emptyset}  $}
  \item {$x\,\in\,\bigcup_{[y]\cap A\neq \emptyset} [y]  $ or $x\,\in\,\bigcup_{[y]\cap
B\neq \emptyset} [y]  $}
  \item {$x\in A^{u}\cup B^{u}$.}
 \end{impemize}}
\item[u-Cap] {By monotonicity, $(A\cap B)\,\subseteq\, A^{u}$ and $(A\cap B)\,\subseteq\,
B^{u}$, so $(A\cap B)^{u}\,\subseteq\, A^{u}\cap B^{u} $.}
\item[Dual]{If $z\in A^{lc}$, then $ z\in [x]^{c}$ for all $[x]\subseteq A$ and either,
$z\in A\setminus A^{l}$ or $z\in A^{c}$. If $z\in A^{c}$ then $z\in A^{cu}$.
If $z\in A\setminus A^{l}$ and $z\neq A^{cu \setminus A^{c}}$ then $[z]\cap A^{c}\,=\,
\emptyset$. But this contradicts $z\notin A^{cu} \setminus A^{c}$.
So $(\forall A\in \wp (S))\, A^{lc}\,\subseteq\, A^{cu} .$}
\end{description}
\end{proof}

\section{Algebras of Rough Definite Elements}

In this section we prove key results on the fine structure of definite elements.

\begin{theorem}
On the set of proto definite elements $\delta_{lu}(S)$ of a \textsf{PRAX} $S$, we can
define the following:
\begin{enumerate}
\item {$x\wedge y\,\stackrel{\Delta}{=}\, x\cap y$}
\item {$x\vee y\,\stackrel{\Delta}{=}\, x\cup y$}
\item {$0\,\stackrel{\Delta}{=}\, \emptyset$, }
\item {$1\,\stackrel{\Delta}{=}\,S$ } 
\item {$x^{c}\,\stackrel{\Delta}{=}\,S \setminus x$}
\end{enumerate}
\end{theorem}

\begin{proof}
We need to show that the operations are well defined. Suppose  $x,\,y$ are proto-definite
elements, then
\begin{enumerate}
\item {\[(x\cap y)^{u}\subseteq x^{u}\cap y^{u} \,=\, x\cap y.\] 
\[(x\cap y)^{l}\,=\, (x^{u}\cap y^{u})^{l} \,=\, (x \cap y)^{ul} = (x\cap y)^{u}
\,=\, x\cap y .\] Since $a^{ul} = a^{u}$ for any $a$.}
\item {\[(x\cup y)^{u}\,=\, x\cup y\,=\, x^{l}\cup y^{l}\subseteq (x\cup y)^{l}.\]}
\item {$0\,\stackrel{\Delta}{=}\, \emptyset$ is obviously well defined.}
\item {Obvious.}
\item {Suppose $A\in \delta_{lu}(S)$, then $(\forall z\in A^{c}) \,[z]\cap A \,=\,
\emptyset$ is essential, else $[z]$ would be in $A^{u}$. This means $[z]\subseteq A^{c}$
and so $A^{c}\,=\, A^{cl}$. If there exists a $a\in A$ such that $[a]\cap A^{c}\neq
\emptyset$, then $[a]\subseteq A^{u}\,=\, A$. So $A^{c}\in \delta_{lu}(S)$. }
\end{enumerate}
\qed
\end{proof}

\begin{theorem}
The algebra $\delta_{proto}(S)\,=\, \left\langle \delta_{lu}(S), \vee,\wedge, c , 0, 1
\right\rangle $ is a Boolean lattice.
\end{theorem}

\begin{proof}
Follows from the previous theorem. The lattice order can be defined via,
$x\leq y$ \ifof \, $\,x\cup y = y$ and $x\cap y = x$.
\qed
\end{proof}

\section{The Representation of Roughly Equal Elements}

The representation of roughly equal elements in terms of definite elements are well known
in case classical rough set theory. In case of more general spaces including tolerance
spaces \cite{AM240}, most authors have been concerned with describing the interaction of
rough approximations of different types and not of the interaction of roughly equal
objects. Higher order approaches, developed by the present author as in \cite{AM99} for
bitten approximation spaces, permit constructs over sets of roughly equal objects. In the
light of the contamination problem \cite{AM99,AM240}, it would be an improvement to
describe without higher order constructs. In this section a new method of representing
roughly equal elements based on expanding concepts of definite elements is developed. 

\begin{definition}
A subset $A$ of $\wp (S)$ will be said to a set of \emph{roughly equal} elements \ifof 
\[(\forall x, y\in A)\, x^{l}\,=\, y^{l}\,\& \, x^{u}\,=\,y^{u}.\] 
It will be said to be \emph{full} if no other subset properly including $A$ has the
property. 
\end{definition}

Relative the situation for a general \textsf{RYS}, we have
\begin{theorem}[Meta-Theorem]
In a \textsf{PRAX} $S$, full set of roughly equal elements is necessarily a union of
intervals in $\wp(S)$.
\end{theorem}

\begin{definition}
A non-empty set of non singleton subsets $\alpha\,=\, \{x\,:\,x\subseteq \wp(S)\}$ will be
said to be a \emph{upper broom} \ifof all of the following hold:
\begin{mitemize}
\item {$(\forall x, y \in \alpha)\, x^{u}\,=\, y^{u} $}
\item {$(\forall x, y \in \alpha)\, x\,\parallel \, y$}
\item {If $\alpha \subset \beta$, then $\beta$ fails to satisfy at least
one of the two
conditions.}
\end{mitemize}
The set of upper brooms of $S$ will be denoted by $\pitchfork (S)$.
\end{definition}

\begin{definition}
A non-empty set of non singleton subsets $\alpha\,=\, \{x\,:\,x\subseteq \wp(S)\}$ will be
said to be a
\emph{lower broom} \ifof all of the following hold:
\begin{mitemize}
\item {$(\forall x, y \in \alpha)\, x^{l}\,=\, y^{l}\neq x $}
\item {$(\forall x, y \in \alpha)\, x\,\parallel \, y$}
\item {If $\beta \subset \alpha$ and $\beta$ has at least two elements,
then $\beta$ fails to satisfy at least one of the two
conditions.}
\end{mitemize}
The set of lower brooms of $S$ will be denoted by $\psi (S)$.
\end{definition}

\begin{proposition}
If $x\in \delta_{lu}(S)$ then $\{x\}\notin \pitchfork (S)$ and $\{x\}\notin \psi (S)$.
\end{proposition}

In the next definition, the concept of union of intervals in a partially ordered set is
modified in a way for use with specific types of objects.  

\begin{definition}
By a \emph{bruinval}, we will mean a subset of $\wp (S)$ of one of the following forms:
\begin{mitemize}
\item {Bruinval-0: Intervals of the form $(x, y),\,[x, y),\, [x, x],\, (x, y]$ for
$x,y\in \wp(S)$.}
\item {Open Bruinvals: Sets of the form $[x, \alpha)\,=\, \{z\, : \, x\, \leq\, z\,
< \, b\,\&\, b\in \alpha\}$, $(x, \alpha]\,=\, \{z\, : \, x\, < \, z\,
\leq\, b\,\&\, b\in \alpha\}$ and $(x, \alpha)\,=\, \{z\, : \, x\, < \, z\,
< \, b\,, b\in \alpha\}$ for $\alpha \in \wp(\wp (S))$.}
\item {Closed Bruinvals: Sets of the form $[x, \alpha]\,=\, \{z\, : \, x\, \leq\, z\,
\leq\, b\,\&\, b\in \alpha\}$ for $\alpha \in \wp(\wp (S))$.}
\item {Closed Set Bruinvals: Sets of the form $[\alpha, \beta]\,=\,\{z\, : \, x\, \leq\,
z\,\leq\, y\,\&\, x\in \alpha \& y\in \beta\}$ for $\alpha , \beta \in \wp(\wp (S))$ }
\item {Open Set Bruinvals: Sets of the form $(\alpha, \beta)\,=\,\{z\, : \, x\, < \,
z\,< \, y\,, x\in \alpha \& y\in \beta\}$ for $\alpha , \beta \in \wp(\wp (S))$.}
\item {Semi-Closed Set Bruinvals: Sets of the form $[[\alpha, \beta]]$ defined as
follows:  $\alpha = \alpha_{1}\cup\alpha_{2} $, $\beta = \beta_{1}\cup\beta_{2}$ and
$[[\alpha,\beta]]\,=\,(\alpha_{1},\beta_{1})\cup
[\alpha_{2},\beta_{2}]\cup(\alpha_{1},\beta_{2}]\cup [\alpha_{2},\beta_{1})  $ for $\alpha
, \beta \in \wp(\wp (S))$.} 
\end{mitemize}
\end{definition}

In the example of the second section, the representation of the rough object
$(P_{i}^{l}, P_{i}^{u})$ requires set bruinvals.

\begin{proposition}
If $S$ is a \textsf{PRAX}, then a set of the form $[x, y]$ with $x, y\in \delta_{lu} (S)$
will be a set of roughly equal subsets of $S$ \ifof $x\,=\, y$. 
\end{proposition}

\begin{proposition}
A bruinval-0 of the form $(x, y)$  is a full set of roughly equal
elements if 
\begin{mitemize}
\item {$x, y \in \delta_{lu} (S)$,}
\item {$x$ is covered by $y$ in the order on $\delta_{lu} (S)$.}
\end{mitemize}
\end{proposition}

\begin{proposition}
If $x, y\in \delta_{lu} (S)$ then sets of the form $[x, y),\, (x,
y]$ cannot be a non-empty set of roughly equal elements, while those of the form $[x, y]$
can be \ifof $x = y$. 
\end{proposition}

\begin{proposition}
A bruinval-0 of the form $[x, y)$  is a full set of roughly equal
elements if 
\begin{mitemize}
\item {$x^{l}, y^{u}\in \delta_{lu} (S)$, $x^{l} = y^{l}$ and $x^{u} = y^{u}$,}
\item {$x^{l}$ is covered by $y^{u}$ in $\delta_{lu} (S)$ and}
\item {$x\setminus (x^{l})$ and $y^{u}\setminus y$ are singletons}
\end{mitemize}
\end{proposition}

\begin{remark}
In the above proposition the condition $x^{l}, y^{u}\in \delta_{lu} (S)$, is not
necessary.  
\end{remark}

\begin{theorem}
If a bruinval-0 of the form $[x, y]$ satisfies 
\begin{mitemize}
\item {$x^{l} = y^{l} = x$ and $x^{u} = y^{u}$,}
\item {$y^{u}\setminus y$ is a singleton.}
\end{mitemize}
then  $[x, y]$ is a full set of roughly equal objects.
\end{theorem}

\begin{proof}
Under the conditions, if $[x, y]$ is not a full set of roughly equal objects, then there
must exist at least one set $h$ such that $h^{l} = x$ and $h^{u} = y^{u}$ and $h\notin
[x, y]$. But this contradicts the order constraint $x^{l}\, \leq \, h \, y^{u}$. Note
that $y^{u}\notin [x, y]$ under the conditions.
\qed 
\end{proof}

\begin{theorem}
If a bruinval-0 of the form $(x, y]$ satisfies 
\begin{mitemize}
\item {$x^{l} = y^{l} = x$ and $(\forall z\in (x, y])\,z^{u} = y^{u}$,}
\item {$y^{u}\setminus y$ is a singleton.}
\end{mitemize}
then $(x, y]$ is a full set of roughly equal objects, that does not intersect the full set
$[x, x^{u}]$.
\end{theorem}

\begin{proof}
By monotonicity it follows that $(x, y]$ is a full set
of roughly equal objects. then there
must exist at least one set $h$ such that $h^{l} = x$ and $h^{u} = y^{u}$ and $h\notin
[x, y]$. But this contradicts the order constraint $x^{l}\, \leq \, h \, y^{u}$. Note
that $y^{u}\notin [x, y]$ under the conditions.
\qed 
\end{proof}

\begin{theorem}
A bruinval-0 of the form $(x^{l}, x^{u})$ is not always a set of roughly equal
elements, but will be so when $x^{uu}= x^{u}$. In the latter situation it will be full if
$[x^{l},x^{u})$ is not full. 
\end{theorem}

The above theorems essentially show that the description of rough objects depends on too
many types of sets and the order as well. Most of the considerations extend to other
types of bruinvals as is shown below and remain amenable. 

\begin{theorem}
An open bruinval of the form $(x, \alpha)$ is a full set of roughly equal elements \ifof
\begin{mitemize}
\item {$\alpha$ is an upper broom.}
\item {$(\forall y\in \alpha)\, x^{l}\,=\, y^{l}, x^{u}\,=\, y^{u}$}
\item {$(\forall z) (x^{l}\subseteq z \subset x \longrightarrow  z^{u}\subset x^{u})$.}
\end{mitemize}
\end{theorem}

\begin{proof}
It is clear that for any $y\in \alpha$, $(x,y)$ is a convex interval and all elements in
it have same upper and lower approximations. The third condition ensures that $[z,\alpha)$
is not a full set for any $z\in [ x^{l}, x)$.
\qed 
\end{proof}

\begin{definition}
An element $x\in \wp (S)$ will be said to be a \emph{weak upper critical element relative}
$z\subset x$ \ifof $(\forall y\in \wp (S))\, (z\,=\,y^{l} \,\&\, x \subset y\,
\longrightarrow x^{u}\subset y^{u} )$.

An element $x\in \wp (S)$ will be said to be an \emph{upper critical element relative}
$z\subset x$ \ifof
$(\forall v,\,y\in \wp (S))\, (z\,=\,y^{l}\,=\,v^{l} \,\&\,v \subset x \subset y\,
\longrightarrow \,v^{u} = x^{u}\subset
y^{u} )$. Note that the inclusion is strict. 

An element $a$ will be said to be \emph{bicritical relative} $b$ \ifof $(\forall x,\,
y\in \wp (S)) (a\subset x \subseteq y \subset b\,\longrightarrow\, x^{u} = y^{u}\,\&\,
x^{l} = y^{l}\,\&\, x^{u} \subset b^{u}\,\&\, a^{l}\subset x^{l})$. 
\end{definition}
If $x$ is an upper critical point relative $z$, then $[z,x)$ or $(z,x)$ is a set of
roughly equivalent elements.

\begin{definition}
An element $x\in \wp (S)$ will be said to be an \emph{weak lower critical element
relative} $z\supset x$ \ifof
$(\forall y\in \wp (S))\, (z\,=\,y^{u}\,\&\,y \subset x\, \longrightarrow \,y^{l}\subset
x^{l} )$.

An element $x\in \wp (S)$ will be said to be an \emph{lower critical element
relative} $z\supset x$ \ifof
$(\forall y, v \in \wp (S))\, (z\,=\,y^{u}\,=\, v^{u}\&\,y \subset x \subset v\,
\longrightarrow \,y^{l}\subset
x^{l}= v^{l} )$.

An element $x\in \wp (S)$ will be said to be an \emph{lower critical element} \ifof
$(\forall y\in \wp (S))\, (y \subset x\, \longrightarrow \,y^{l}\subset x^{l} )$
An element that is both lower and upper critical will be said to be \emph{critical}.
The set of upper critical, lower critical and critical elements respectively will be
denoted by $UC(S)$, $LC(S)$ and $CR(S)$. 
\end{definition}

\begin{proposition}
In a \textsf{PRAX}, every upper definite subset is also upper critical, but the converse
need not hold. 
\end{proposition}

The most important thing about the different lower and upper critical points is that they
help in determining full sets of roughly equal elements by determining the boundaries of
intervals in bruinvals of different types. 

\section{On Atoms in the POSET of Roughly Equivalent Sets}
 
\begin{definition}
 For any two elements $x, y\in \wp(S)|\approx$, let 
\[x\,\leq\, y \,\ifsf \, (\forall a\in x)(\forall b\in y) a^{l}\subseteq
b^{l}\,\&\,a^{u}\subseteq b^{u}. \]
$\wp(S)|\approx$ will be denoted by $H$ in what follows.
\end{definition}

\begin{proposition}
The relation $\leq$ defined on $H$ is a bounded and directed partial order. The least
element will be denoted by $0$ ($0\,=\, \{\emptyset\}$) and the greatest by $1$
($1\,=\,\{S\}$).  
\end{proposition}

\begin{definition}
For any $a, b\in H$, let $UB(a, b)\,=\, \{x\, :\,a\leq x \,\&\,b\leq x \}$ and 
$LB(a, b)\,=\, \{x\, :\,x\leq a \,\&\,x\leq b \}$. By a \emph{s-ideal} (strong ideal) of
$H$, we will mean a subset $K$ that satisfies all of 
\begin{mitemize}
\item {$(\forall x\in H)(\forall a\in K)(x\leq a\,\longrightarrow\, x\in K) $,}
\item {$(\forall a, b\in K) UB(a, b)\cap K\neq \emptyset $.} 
\end{mitemize}

An \emph{atom} of $H$ is any element that covers $0$. The set of atoms of $H$ will be
denoted by $At(H)$. 
\end{definition}

\begin{theorem}
Atoms of $H$ will be of one of the following types:
\begin{description}
\item [Type-0]{Elements of the form $(\emptyset, [x])$, that intersect no other set of
roughly equivalent sets. }
\item [Type-1]{Bruinvals of the form $(\emptyset, \alpha)$, that do not contain full sets
of roughly equivalent sets.}
\item [Type-2]{Bruinvals of the form $(\alpha, \beta)$, that do not contain full sets of
roughly equivalent sets and are such that $(\forall x) x^{l}\,=\, \emptyset$.}
\end{description}
\end{theorem}

\begin{proof}
It is obvious that a bruinval of the form $(\alpha, \beta)$ can be an atom only if
$\alpha$ is the $\emptyset$. If not, then each element $x$ of the bruinval $(\emptyset,
\alpha)$ will satisfy $x^{l}=\emptyset \,\subset \, x^{u}$, thereby contradicting the
assumption that $(\alpha, \beta)$ is an atom.

If $[x]$ intersects no other successor neighbourhood, then \[(\forall y\in (\emptyset,
[x])) y^{l}\,=\, \emptyset\,\&\,x^{u}\,=\, [x]\] and it will be a minimal set of
roughly equal elements containing $0$. 

The other part can be verified based on the representation of possible sets of roughly  
equivalent elements.
\qed 
\end{proof}

\begin{theorem}
The partially ordered set $H$ is atomic.
\end{theorem}

\begin{proof}
We need to prove that any element $x$ greater than $0$ is either an atom or  
there exists an atom $a$ such that $a \leq x$, that is 
\[(\forall x)(\exists a\in At(H))(0 < x\,\longrightarrow\, a\leq x ) .\]

Suppose the bruinval $(\alpha,\beta)$ represents a non-atom, then it is necessary that 
\[(\forall x\in \alpha)\, x^{l} \neq \emptyset \,\&\, x^{u}\subseteq S .  \]

Suppose the neighbourhoods included in $x^{u}$ are $\{[y]\,:\, y\in B\subseteq S\}$.
If all combinations of bruinvals of the form $(\emptyset,\gamma)$ formed from these
neighbourhoods are not atoms, then it is necessary that the upper approximation of every
singleton subset of a set in $\gamma$ properly contains another non-trivial
upper approximation. This is impossible. 

So $H$ is atomic. 
\qed
\end{proof}

\section{Geometry of Granular Knowledge Interpretation}

In my opinion, \emph{Any knowledge, however involved, may be seen as a collection of
concepts with admissible operations of reasoning defined on them}. Knowledges associated
\textsf{PRAX} have various peculiarities corresponding to the
semantic evolution of rough objects in it. The semantic domains of representation
properly contain the semantic domains of interpretation. Not surprisingly, it is
because the rough objects corresponding to $l,\, u$ cannot be represented perfectly in
terms of objects from $\delta_{lu}(S)$ alone. \emph{In the nongranular perspective too,
this representation aspect should matter} - ''\emph{should}'', because it is matter of
choice during generalization from the classical case in the non granular approach.   

The natural \rsds  of $l,\, u$ is Meta-R, while that of $l_{o},\, u_{o}$ is
$\mathfrak{O}$ (say, corresponding rough objects of $\tau (R)$). These can be seen as
separate domains or as parts of a minimal containing domain that permits enough
expression. The main problem of granular \textsf{KI} is that knowledge is
correctly representable in terms of \emph{atomic concepts of knowledge} at semantic
domains placed between Meta-C and Meta-R and not at the latter. So the
characterization of possible semantic domains and their mutual ordering - leading to their
geometry is of interest.

The following will be assumed to be \emph{part} of the interpretation:
\begin{mitemize}
\item {Two types of rough objects corresponding to Meta-R and $\mathfrak{O}$  and their
natural correspondence correspond to concepts or weakenings thereof. A concept relative
one semantic domain need not be one of the other.}
\item {A granule of the \rsd $\mathcal{O}$ is necessarily a concept of $\mathcal{O}$, but
a granule of Meta-R may not be a concept of $\mathcal{O}$ or Meta-R.}
\item {Critical points are not necessarily concepts of either semantic domain.}
\item {Critical points and the representation of rough objects require the \rsds to be
extended.}
\end{mitemize}

The above obviously assumes that a \textsf{PRAX} $S$ has at least two kinds of knowledge
associated (in relation to the Pawlak-sense interpretation). To make the interpretations
precise, we will indicate them by $\mathcal{I}_{1}(S)$ and $\mathcal{I}_{o}(S)$
respectively (corresponding to the approximations to $l,\, u$ and $l_{o},\, u_{o}$
respectively). The pair $(\mathcal{I}_{1}(S),\,\mathcal{I}_{o}(S))$ will also be
referred to as the \emph{generalized \textsf{KI}}. 
\begin{definition}
Given two \textsf{PRAX} $S = \left\langle \underline{S},\,R \right\rangle$, $
V = \left\langle \underline{S},\,Q \right\rangle$, $S$ will be said to be \emph{o-coarser}
than $V$ \ifof $\mathcal{I}_{o}(S)$ is coarser than $\mathcal{I}_{o}(V)$ in Pawlak-sense
( that is $\tau(R)\subseteq \tau (Q)$). Conversely, $V$ will be said to be a
\emph{o-refinement} of $S$.

$S$ will be said to be \emph{p-coarser} than $V$ \ifof 
$\mathcal{I}_{1}(S)$ is coarser than $\mathcal{I}_{1}(V)$ in the sense $R\subseteq Q$.
Conversely, $V$ will be said to be a \emph{p-refinement} of $S$.  
\end{definition}

An extended concept of positive regions is defined next.

\begin{definition}
If $S_{1}\,=\, \left\langle \underline{S},Q \right\rangle $ and $S_{2}\,=\, \left\langle
\underline{S},P \right\rangle $ are two \textsf{PRAX} such that $Q\subset R$, then
by the
\emph{granular positive region} of $Q$ with respect to $R$ is given by $gPOS_{R}(Q)\,=\,
\{ [x]_{Q}^{l_{R}}\, :\, x\in \underline{S} \}$, where $[x]_{Q}^{l_{R}}$ is the lower
approximation (relative $R$) of the $Q$-related elements of $x$. Using this we can define
the \emph{granular extent of dependence} of knowledge encoded by $R$ on the knowledge
encoded by $Q$ by natural injections $:gPOS_{R}(Q) \longmapsto \mathcal{G}_{R} $.  
 \end{definition}

Lower critical points can be naturally interpreted as preconcepts that are
definitely included in the discourse, while upper critical points are preconcepts that
include most of the discourse. The problem with this interpretation is that it's
representation requires a semantic domain at which critical points of different kinds can
be found. A key requirement for such a domain would be the feasibility of rough counting
procedures like \textsf{IPC} \cite{AM240}. We will refer to a semantic domain that has
critical points of different types as basic objects as a \textsf{Meta-RC}. 

The following possible axioms of granular knowledge that also figure in my earlier paper
\cite{AM909}, get into difficulties with the present approach and even when we restrict
attention to $\mathcal{I}_{1}(S)$: 
\begin{enumerate}
\item {Individual granules are atomic units of knowledge.}
\item {Maximal collections of granules subject to a concept of mutual
independence are admissible concepts of knowledge.}
\item {Parts common to subcollections of maximal collections of granules are also
knowledge.}
\end{enumerate}
The first axiom holds in weakened form as the granulation $\mathcal{G}$ for
$\mathcal{I}_{1}(S)$ is only lower definite and affects the other. The possibility of
other nice granulations being possible for the \textsf{PRAX} case appears to be possible
at the cost of other nice properties. So we can conclude that in proper 
\textsf{KR} happens at semantic domains like \textsf{Meta-RC} where critical
points of different types are perceived. Further at Meta-R, rough objects may correspond
to knowledge or conjectures - if we require the concept of proof to be an ontological
concept or beliefs. The scenario can be made more complex with associations
of $\mathfrak{O}$ knowledges.

From a non-granular perspective, in Meta-R  rough objects must
correspond to knowledge with some of them lacking a proper evolution - there is no problem
here. Even if we permit $\mathfrak{O}$ objects, then in the perspective we would be able
to speak of two kinds of closely associated knowledges. 

\section {Further Directions and Conclusion}

In this research we have developed a new general rough set theory over proto transitive
relations, the representation of rough objects and definite objects. The knowledge
interpretation of rough sets is also generalized to provide sensible knowledge
interpretations across different semantic domains. This paves the way for possible
semantics, measures and/or logics of knowledge consistency. We have also shown in concrete terms 
that granular knowledge interpretation and classical
knowledge interpretation are very different things.

\bibliographystyle{splncs.bst}
\bibliography{biblioam6.bib}
\end{document}